\documentclass{article}[a4paper]
\usepackage[margin=1in]{geometry}
\usepackage{graphicx}%
\usepackage{multirow}%
\usepackage{amsmath,amssymb,amsfonts}%
\usepackage{amsthm}%
\usepackage{mathrsfs}%
\usepackage{bbm}
\usepackage[title]{appendix}%
\usepackage{xcolor}%
\usepackage{textcomp}%
\usepackage{manyfoot}%
\usepackage{booktabs}%
\usepackage{algorithm}%
\usepackage{algorithmicx}%
\usepackage{algpseudocode}%
\usepackage{listings}%
\usepackage{hyperref}

\usepackage{subcaption}

\newtheorem{theorem}{Theorem}
\newtheorem{lemma}{Lemma}
\theoremstyle{definition}

\newcommand{\argmin}{\mathop{\mathrm{argmin}}}

\usepackage[backend=biber]{biblatex}
\providecommand{\keywords}[1]
{
  \textbf{\textit{Key words: }} #1
}

\providecommand{\subjclass}[1]
{
  \textbf{\textit{MSC2020: }} #1
}

\begin{document}

\title{Random feature-based double Vovk-Azoury-Warmuth algorithm for online multi-kernel learning}

\author{
  Dmitry B. Rokhlin\thanks{The research of D.B.\,Rokhlin was supported by the Regional Mathematical Center of the Southern Federal University with the Agreement no. 075-02-2025-1720 of the Ministry of Science and Higher Education of Russia.} \\
  Institute of Mathematics, Mechanics and Computer Sciences of the\\
  Southern Federal University \\
  Regional Scientific and Educational Mathematical Center of the\\
  Southern Federal University \\
  \texttt{dbrohlin@sfedu.ru}
  \and
  Olga V. Gurtovaya \\
  Institute of Mathematics, Mechanics and Computer Sciences of the\\
  Southern Federal University \\
   \texttt{imedashvili@sfedu.ru}
}

 \date{} 

\maketitle

\abstract{We introduce a novel multi-kernel learning algorithm, VAW$^2$, for online least squares regression in reproducing kernel Hilbert spaces (RKHS). VAW$^2$ leverages random Fourier feature-based functional approximation and the Vovk-Azoury-Warmuth (VAW) method in a two-level procedure: VAW is used to construct expert strategies from random features generated for each kernel at the first level, and then again to combine their predictions at the second level. A theoretical analysis yields a regret bound of $O(T^{1/2}\ln T)$ in expectation with respect to artificial randomness, when the number of random features scales as $T^{1/2}$. Empirical results on some benchmark datasets demonstrate that VAW$^2$ achieves superior performance compared to the existing online multi-kernel learning algorithms: Raker and OMKL-GF, and to other theoretically grounded method methods involving convex combination of expert predictions at the second level.}



\section{Introduction}\label{sec:1}
Kernel methods \cite{Scholkopf2002,Hofmann2008} allow to extend the scope of the linear models to the analysis of complex nonlinear dependencies by working in reproducing kernel Hilbert spaces (RKHS). They combine high expressive power, formalized as the universality property (see, e.g. \cite{Sriperumbudur2011}), with the possibility of using tools from the convex analysis to establish global optimality results. However, the computational complexity of these methods grows as $T^3$, where $T$ is the number of examples in a classical batch supervised learning problem. 

The gradient descent algorithm, being applied in an RKHS in the online mode \cite{Kivinen2004online}, at each iteration increases the complexity of the linear combination of kernels by adding a new ``support vector'' (SV) to a dictionary. There is a lot of techniques for dealing with this phenomenon, called the curse of kernelization \cite{Wang2012}. These techniques can be broadly categorized into budget maintenance strategies and functional approximation strategies \cite{Hoi2021online}. The budget maintenance strategies include SV removal, SV projection and SV merging families of algorithms. Similar kernel adaptive filtering algorithms were developed for signal processing \cite{Slavakis2014,VanVaerenbergh2014online}. 

In this paper, we follow the functional approximation strategy \cite{Lu2016large} based on random Fourier features (RFF) \cite{Rahimi2007}. This approach allows to work in a fixed dimensional space in each iteration. However, to get sublinear regret with respect to a ball in an RKHS, the dimension of this space, which is equal to the number of random features, should grow with $T$. 

Besides the computational complexity, another issue with kernel methods concerns the kernel selection, which essentially influences the results. Multi-kernel methods try to address this issue by choosing a kernel combination from a large preselected dictionary \cite{Gonen2011multiple}. In the online learning setting multi-kernel methods in conjunction with the RFF-based functional approximation were used in \cite{Sahoo2019large,Shen2019random}. These papers apply the online gradient descent method to random feature vectors related to each kernel to generated ``expert'' strategies, and then combine their predictions by an exponential weight update rule (used in both papers \cite{Sahoo2019large,Shen2019random}), or by the online gradient descent algorithm (used in \cite{Sahoo2019large}).

In this paper we are interested in the online least squares regression problem in the RKHS spaces. In the finite dimensional case the Vovk-Azoury-Warmuth (VAW) algorithm \cite{Vovk2001,Azoury2001} provides an optimal regret bound $O(\ln T)$ (see also \cite{Gaillard2019uniform}). In the general case the regret w.r.t. a predictor in a ball of an RKHS can be bounded by $O(T^{1/2})$ \cite{Vovk2006}, and this bound is not improvable. 
We consider the problem in the multi-kernel setting and prove a loss bound $O(T^{1/2}\ln T)$ in expectation w.r.t. an artificial randomness. This bound is obtained via computationally feasible algorithms. We apply VAW algorithm for construction of expert strategies and either VAW or exponentially weighted average (EWA) forecasting algorithm for combining expert predictions.  Note that the regret bounds of \cite{Sahoo2019large,Shen2019random} are not applicable due to the lack of the global Lipschitz condition. 

The paper is organized as follows. In Section \ref{sec:2} we recall the definition of an RKHS space and fix a class of RKHS spaces with translation invariant kernels as in \cite{Rahimi2008}. 
We also provide a simple result related to approximation by a linear combination of random features (Lemma \ref{lem:1}), and recall the basic regret bound of the VAW algorithm.

The main results are contained Section \ref{sec:3}. We consider a dictionary, containing $N$ kernels $k_i$ and related RKHS spaces $\mathcal H_i$. For each kernel we generate $m$ random features and apply the VAW algorithm either to the concatenated $Nm$-dimensional vector (Theorem \ref{th:1}), or to each $m$-dimensional vector separately. In the second case we combine prediction of the ``expert'' VAW algorithms either by the VAW algorithm (Theorem 2), or by the EWA algorithm (Theorem 3). The first approach (adopted in Theorem 1) provides the regret bounds w.r.t. elements of a ball in the large RKHS space $\mathcal H$ with the kernel $k=k_1+\dots+k_N$, while the second approach provides the same bound only w.r.t. the elements of a ball in each $\mathcal H_i$. At the same time, the second approach has lower computational and spatial complexity, and we consider it to be the primary one.

In Section \ref{sec:4}, we provide computer experiments on several benchmark datasets. We compare the performance of VAW$^2$ against state-of-the-art online multi-kernel learning algorithms, and to other traditional methods of combining VAW expert predictions. The results demonstrate the effectiveness of VAW$^2$ in achieving superior prediction accuracy. Section \ref{sec:5} concludes.

\section{Preliminaries} \label{sec:2}
Recall that a reproducing kernel Hilbert space (RKHS) is a Hilbert space $\mathcal{H}$ of functions $f: \mathcal{X} \rightarrow \mathbb{R}$ such that any evaluation functional $x\mapsto f(x)$, $x\in\mathcal X$ is bounded. If $\mathcal H$ is an RKHS on $\mathcal X$, then by the Riesz representation theorem for each $x \in\mathcal X$ there exists a unique element, $k_x \in\mathcal H$, such that for every $f\in\mathcal H$, 
\[ f(x) = \langle f, k_x\rangle_\mathcal H. \]
The function $k:\mathcal X\times\mathcal X\to\mathbb R$ defined by
$ k(x,x')=k_x(x')$
is called the reproducing kernel of $\mathcal H$. The kernel can be expressed via the feature map $x\mapsto k_x$:
$ k(x,x')=\langle k_x,k_{x'}\rangle_{\mathcal H}.$

Consider a continuous function $\phi:\mathbb R^d\times \Theta\to [-a,a]$, where $\Theta$ is a closed subset of a finite dimensional space. Following \cite{Rahimi2008}, we will consider only reproducing kernel Hilbert spaces of the form
\begin{align} \label{2.1}
\mathcal H=\left\{x\mapsto f(x)=\int_\Theta\alpha(\theta) \phi(x;\theta)\,d\theta: \int_\Theta\frac{\alpha^2(\theta)}{p(\theta)}\,d\theta<\infty \right\} 
\end{align}
with the inner product
\[ \langle f,g\rangle_\mathcal H=\int_\Theta \frac{\alpha(\theta)\beta(\theta)}{p(\theta)}\,d\theta,\]
where $g(x)=\int\beta(\theta) \phi(x;\theta)\,d\theta$. 
In \cite[Proposition 4.1]{Rahimi2008} it is proved that $\mathcal H$ is an RKHS with the reproducing kernel
\begin{equation}  \label{2.2}
k(x,y)=\int_\Theta p(\theta)\phi(x;\theta) \phi(y;\theta)\,d\theta.
\end{equation}
In particular,
\begin{align*}
\langle f,k(x,\cdot)\rangle_\mathcal H &=\left\langle \int_\Theta\alpha(\theta) \phi(\cdot;\theta)\,d\theta, \int_\Theta p(\theta)\phi(x;\theta) \phi(\cdot;\theta)\,d\theta\right\rangle_\mathcal H=\int_\Theta \alpha(\theta)\phi(x;\theta)\,d\theta=f(x).
\end{align*}

Assume that the kernel $k$ is translation invariant: $k(x,y)=\kappa(x-y)$. Then by the Bochner theorem
\[ \kappa(z)=\int_{\mathbb R^d} e^{i\langle\omega,z\rangle}\Lambda(d\omega)\]
for some non-negative $\sigma$-additive measure $\Lambda$ on the Borel $\sigma$-algebra $\mathscr B(\mathbb R^d)$. For our purposes it is enough to 
assume that $\Lambda$ is absolutely continuous w.r.t. the Lebesgue measure:
\[ \kappa(z)=\int_{\mathbb R^d} e^{i\langle\omega,z\rangle} q(\omega)\,d\omega=\int_{\mathbb R^d} q(\omega)\cos\langle\omega,z\rangle \,d\omega,\]
where $q$ is a probability density function. 
In particular, for Gaussian kernels:
\begin{equation} \label{2.2A}
k(x,y)=e^{-\|x-y\|_2^2/(2\sigma^2)},\quad q(\omega)=\left(\frac{\sigma}{\sqrt{2\pi}}\right)^d e^{-\sigma^2\|\omega\|_2^2/2},   
\end{equation}
for Laplacian kernels:
\begin{equation} \label{2.2B}
k(x,y)=e^{-\|x-y\|_1/\sigma},\quad q(\omega)= 
\frac{\sigma^d}{\pi^d}\prod_{j=1}^d\frac{1}{1+\sigma^2 \omega_j^2}.
\end{equation}
We see that here $q$ are products of Gaussian and Cauchy distributions respectively (see \cite{Rahimi2007}).

For such kernels formula (\ref{2.2}) holds true with $\Theta=\mathbb R^d\times [0,2\pi]$, $\theta=(\omega,b)$,
\begin{align*}
 p(\theta)&=q(\omega) r(b),\quad r(b)=1/(2\pi),\\   
\phi(x;\theta)&=\sqrt{2}\cos(\langle \omega,x\rangle+b),    
\end{align*}
(see \cite{Rahimi2007}). We have,
\begin{align*}
\int_\Theta p(\theta)\phi(x;\theta) \phi(y;\theta)\,d\theta &=\frac{1}{2\pi}\int_0^{2\pi}\int_{\mathbb R^d} 2\cos(\langle\omega,x\rangle+b) \cos(\langle\omega,y\rangle+b) q(\omega)d\omega db \\
&= \int_{\mathbb R^d} \cos\langle\omega,x-y\rangle q(\omega) d\omega=\kappa(x-y)=k(x,y).
\end{align*}

Consider the vector $\Phi_\theta(x) = (\phi(x,\theta_k))_{k=1}^{m}=(\sqrt{2}\cos(\langle \omega_k,x\rangle+b_k))_{k=1}^{m}$ 
of random Fourier features, generated from the distributions $p$. Here
$ \omega_k\sim q$, $b_k\sim U(0,2\pi)$
are i.i.d. random variables.  
Denote by $\operatorname*{\mathsf E}_\theta$ the expectation w.r.t. to the joint distribution of $\theta_1,\dots,\theta_m$.

The following simple result shows that any element of $\mathcal H$, defined by (\ref{2.1}), can be approximated by a linear combination of random Fourier features. 
\begin{lemma} \label{lem:1}
For any $f=\int\gamma(\theta)\phi(\cdot,\theta)\,d\theta\in\mathcal H$ put
\[\widehat w=\frac{1}{m}\left(\frac{\gamma(\theta_1)}{p(\theta_1)},\dots,\frac{\gamma(\theta_m)}{p(\theta_m)} \right),\]
where $\theta_i\sim p$ are i.i.d. random variables. Then
\begin{equation} \label{2.3}
\mathsf E_\theta (\langle \widehat w,\Phi_\theta(x)\rangle-f(x))^2\le 2\frac{\|f\|_{\mathcal H}^2}{m},\qquad \mathsf E_\theta\|\widehat w\|_2^2=\frac{\|f\|_{\mathcal H}^2}{m}.
\end{equation}
\end{lemma}
\begin{proof} The random estimate $\langle \widehat w,\Phi_\theta(x)\rangle$ of $f(x)$ is unbiased:
\begin{align} \label{2.3A}
 \mathsf E_\theta\langle \widehat w,\Phi_\theta(x)\rangle=\frac{1}{m}\sum_{i=1}^m \mathsf E_{\theta_i}\left(\frac{\gamma(\theta_i)}{p(\theta_i)}\phi(x,\theta_i) \right)=f(x).    
\end{align}
Compute the variance of this estimate:
\begin{align*}
&\mathsf E_\theta\left(\frac{1}{m}\sum_{k=1}^m \frac{\gamma(\theta_k)}{p(\theta_k)}\phi(x;\theta_k)-f(x) \right)^2=\frac{1}{m}\mathsf E_{\theta_1}\left(\frac{\gamma(\theta_1)}{p(\theta_1)}\phi(x;\theta_1)-f(x)\right)^2\nonumber\\
&\le\frac{1}{m}\mathsf E_{\theta_1}\left(\frac{\gamma(\theta_1)}{p(\theta_1)}\phi(x;\theta_1)\right)^2=\frac{1}{m}\int\frac{\gamma^2(\theta_1)}{p(\theta_1)}\phi^2(x;\theta)\,d\theta_1\le \frac{2}{m}\int\frac{\gamma^2(\theta_1)}{p(\theta_1)}\,d\theta_1\le 2\frac{\|f\|_{\mathcal H}^2}{m}. \label{3.4}
\end{align*}
The proof of the equality in (\ref{2.3}) is also elementary:
\[ \mathsf E_\theta\|\widehat w\|^2_2=\frac{1}{m^2}\sum_{i=1}^m\mathsf E_{\theta_i}\left(\frac{\gamma^2(\theta_i)}{p^2(\theta_i)}\right)=\frac{\|f\|_{\mathcal H}^2}{m}. \qedhere \]
\end{proof}

Let $(x_t,y_t)\in\mathbb R^d\times\mathbb R$ be an arbitrary sequence. Assuming that the dependence between features $x_t$ and labels $y_t$ can be described sufficiently well by a function $f\in\mathcal H$, consider the least squares problem 
\begin{equation} \label{2.4}
\sum_{t=1}^T (y_t-f(x_t))^2\to\min_{f\in\mathcal H}.
\end{equation}
Lemma \ref{lem:1} allows to pass to its parametric form:
\[ \sum_{t=1}^T (y_t-\langle w,\Phi_\theta(x_t) 
\rangle)^2 \to\min_{w\in\mathbb R^m}.\]
More precisely, we will consider the online learning problem, where the goal is to find a sequence $w_t$ with ``small'' cumulative expected loss:
\[ \mathsf E_\theta\sum_{t=1}^T (y_t-\langle w_t,\Phi_\theta(x_t) \rangle)^2,\quad \textrm{where}\quad w_t=w_t(\Phi_\theta(x_1),\dots,\Phi_\theta(x_t),y_1,\dots,y_{t-1}), \]
compared to the loss (\ref{2.4}) of any element $f\in\mathcal H$.

We allow the weight $w_t$ to depend on the feature mapping $\Phi_\theta(x_t)$, indicating that features $x_t$ are available at time $t$ before predicting the label $y_t$. This natural
assumption is important in the Vovk-Azoury-Warmuth (VAW) algorithm \cite[Section 11.8]{Cesa2006prediction}, defined by 
\[ w_t = \argmin_{w\in\mathbb R^d}\left\{\frac{\lambda}{2}\|w\|_2^2 + \frac{1}{2}\sum_{i=1}^{t-1} (\langle \Phi_\theta(x_i), w\rangle - y_i)^2 + \frac{1}{2} \langle \Phi_\theta(x_t),w\rangle^2\right\}.\]
Explicitly,
\begin{align} \label{2.4A}
w_t=S_t^{-1}\sum_{i=1}^{t-1} y_i \Phi_\theta(x_i),\quad S_t=\lambda I_d + \sum_{i=1}^t \Phi_\theta(x_i) \Phi_\theta(x_i)^\top. 
\end{align}
Moreover, $S_t^\top$ can be computed recursively by the Sherman-Morrison formula (which is also presented in \cite{Cesa2006prediction}):
\begin{align} \label{2.4B}
S_t^{-1}=S_{t-1}^{-1}-\frac{S_{t-1}^{-1} \Phi_\theta(x_t) (S_{t-1}^{-1} \Phi_\theta(x_t))^T}{1+\Phi_\theta(x_t)^T S_{t-1}^{-1} \Phi_\theta(x_t)},\quad S_0^{-1}=\lambda^{-1} I_d.
\end{align}

In the sequel, we will assume that the labels $y_t$ are uniformly bounded: $|y_t|\le Y$. The regret \cite{Cesa2006prediction}
\[ R_T(w)=\frac{1}{2}\sum_{t=1}^T(\langle x_t, w_t\rangle-y_t)^2-\frac{1}{2}\sum_{t=1}^T(\langle x_t,w\rangle-y_t)^2\]
of the VAW algorithm satisfies the bound
\begin{equation} \label{2.5}
R_T(w)\le \frac{\lambda}{2}\|w\|_2^2 + \frac{m Y^2 }{2} \ln \left(1+\frac{\rho^2 T}{\lambda m}\right),
\end{equation}
if $\|\Phi_\theta(x_t)\|_2\le \rho$: see \cite[Theorem 11.8]{Cesa2006prediction}, \cite[Theorem 7.34]{Orabona2023}. In our case $\rho=\sqrt{2m}$. Thus,
\begin{equation} \label{2.5A}
R_T(w)\le \frac{\lambda}{2}\|w\|_2^2 + \frac{m Y^2 }{2} \ln \left(1+\frac{2 T}{\lambda}\right).
\end{equation}

\section{Main results} \label{sec:3}
Consider $N$ translation invariant kernels $k_j(x,y)=\kappa_j(x-y)$, $j=1,\dots, N$. Let $\mathcal H_j$ be the correspondent RKHS's. Put $\mathcal H=\mathcal H_1+\dots +\mathcal H_N:=\{f_1+\dots+f_N:f_j\in \mathcal H_j,\ j=1\,\dots,N\}. $
It is known that $\mathcal H$ with the norm
\[ \|f\|_\mathcal H^2=\min\left\{\sum_{j=1}^N \|f_j\|^2_{\mathcal H_j}: f=\sum_{j=1}^N f_j\right\}\]
is an RKHS with the kernel $k=k_1+\dots+k_N$ \cite[Proposition 12.27]{Wainwright2019}. Denote by $B_R(\mathcal H)=\{f\in\mathcal H:\|f\|_{\mathcal H}\le R\}$ the $R$-ball in an RKHS $\mathcal H$.

\begin{lemma} \label{lem:2}
For $f\in\mathcal H=\mathcal H_1+\dots+\mathcal H_N$ take $f_j=\int_\Theta \gamma_j(\theta)\phi_j(x;\theta)d\theta\in \mathcal H_j$ such that 
\[ f=\sum_{j=1}^N f_j,\quad\ \|f\|_\mathcal H^2=\sum_{j=1}^N \|f_j\|^2_{\mathcal H_j}.\]
For each kernel $k_j$ define  
\[\widehat w_j=\frac{1}{m}\left(\frac{\gamma_j(\theta_{j1})}{p_j(\theta_{j1})},\dots,\frac{\gamma_j(\theta_{jm})}{p_j(\theta_{jm})} \right),\]
as in  Lemma \ref{lem:1}. Here  $\theta_{jk}\sim p_j$, $k=1,\dots,m$ are i.i.d. random variables for each $j=1,\dots,N$. Let $|y|\le Y$. Then
\begin{align}
  \mathsf E_\theta \left( \sum_{j=1}^N\langle \widehat w_j,\Phi_{\theta_j}(x)\rangle-y\right)^2 &\le 2\frac{N}{m}\|f\|_{\mathcal H}^2+\left(f(x)-y\right)^2, \label{3.1}
\end{align}
where $\Phi_{\theta_j}(x)=(\phi(x,\theta_{jk}))_{k=1}^m$.
\end{lemma}
\begin{proof} Since the estimate $\langle \widehat w_j,\Phi_{\theta_j}(x)\rangle$ of $f_j(x)$ is unbiased: see (\ref{2.3A}), we have 
\begin{align}
  \mathsf E_\theta \left( \sum_{j=1}^N\langle \widehat w_j,\Phi_{\theta_j}(x)\rangle-y\right)^2  - \left( \sum_{j=1}^N f_j(x)-y\right)^2&=\mathsf E_\theta \left( \sum_{j=1}^N\langle \widehat w_j,\Phi_{\theta_j}(x)\rangle\right)^2-\left( \sum_{j=1}^N f_j(x)\right)^2\nonumber\\
  &=\mathsf E_\theta \left( \sum_{j=1}^N\langle \widehat w_j,\Phi_{\theta_j}(x)\rangle-\sum_{j=1}^N f_j(x)\right)^2 \label{3.2}
\end{align}
Using the inequality $(\sum_{i=1}^N a_i)^2\le N\sum_{i=1}^N a_i^2$, by Lemma \ref{lem:1} we get
\begin{align} \label{3.3}
\mathsf E_\theta\left( \sum_{j=1}^N\langle \widehat w_j,\Phi_{\theta_j}(x)\rangle-f_j(x)\right)^2\le N\sum_{j=1}^N\mathsf E_\theta \left(\langle \widehat w_j,\Phi_{\theta_j}(x)\rangle-f_j(x)\right)^2 \le 2\frac{N}{m}\sum_{j=1}^N\|f_j\|_{\mathcal H_j}^2=2\frac{N}{m}\|f\|^2_{\mathcal H}.    
\end{align} 
The inequalities (\ref{3.2}), (\ref{3.3}) imply (\ref{3.1}).
\end{proof}

Let us first apply the VAW algorithm to the sequence $(\Phi_\theta(x_t),y_t)$, where
\begin{equation} \label{3.7A}
\Phi_\theta(x)=(\Phi_{\theta_1}(x),\dots,\Phi_{\theta_N}(x)), \quad \Phi_{\theta_j}(x)=(\phi_j(x,\theta_{jk}))_{k=1}^m.  
\end{equation}
That is, we concatenate random feature vectors $\Phi_{\theta_j}$, related to each kernel $k_j$, into a $Nm$-dimensional vector. 
\begin{theorem} \label{th:1}
Let $w_t=(w_{t,1},\dots,w_{t,Nm})\in\mathbb R^{Nm}$ be generated by the VAW algorithms applied to the sequence $(\Phi_\theta(x_t),y_t)$. Then
\begin{align}
 \frac{1}{2}\mathsf E_\theta\sum_{t=1}^T (\langle w_t,\Phi_\theta(x_t)\rangle-y_t)^2 &\le \frac{1}{2}\sum_{t=1}^T\left(f(x_t)-y_t\right)^2+\left(\frac{\lambda}{2}+NT\right)\frac{\|f\|_{\mathcal H}^2}{m} +
 \frac{NmY^2}{2}\ln\left(1+\frac{2T}{\lambda} \right) \label{3.8}
\end{align}
for any $f$ in the  RKHS $\mathcal H$, generated by the kernel $k=k_1+\dots+k_N$.
For $T\to+\infty$,
\begin{align}
 \frac{1}{2}\mathsf E_\theta\sum_{t=1}^T (\langle w_t,\Phi_\theta(x_t)\rangle-y_t)^2 &\le \frac{1}{2}\inf_{f\in B_R(\mathcal H)}\sum_{t=1}^T\left(f(x_t)-y_t\right)^2 \nonumber\\
 &+O\left(N(R^2+Y^2\ln T)\sqrt T \right),\quad \textrm{if}\quad   m\propto \sqrt{T}. \label{3.10}   
\end{align}
\end{theorem}
\begin{proof} Denote by $R_T^{\operatorname{VAW}}(w_1,\dots,w_N)$ the regret of the VAW algorithm w.r.t. the fixed vector $(w_1,\dots,w_N)\in\mathbb (\mathbb R^m)^N$. For $f\in\mathcal H$ take $f_j$, $\widehat w_j$ as in Lemma \ref{lem:2}. Then
\begin{align*} 
 \frac{1}{2}\sum_{t=1}^T (\langle w_t,\Phi_\theta(x_t)\rangle-y_t)^2 &= R_T^{\operatorname{VAW}}(\widehat w_1,\dots,\widehat w_N)+\frac{1}{2}\sum_{t=1}^T\left( \sum_{j=1}^N\langle \widehat w_j,\Phi_{\theta_j}(x_t)\rangle-y_t\right)^2.
\end{align*}
By (\ref{2.5A}) and Lemma \ref{lem:1},
\begin{equation} \label{3.13}
\mathsf E_\theta R_T^{\operatorname{VAW}}(\widehat w_1,\dots,\widehat w_N)\le \frac{\lambda}{2}\sum_{j=1}^N \mathsf E_\theta \|\widehat w_j\|_2^2+\frac{NmY^2}{2}\ln\left(1+\frac{2T}{\lambda} \right)\le \frac{\lambda}{2}\frac{\|f\|_{\mathcal H}^2}{m}+\frac{NmY^2}{2}\ln\left(1+\frac{2T}{\lambda} \right).   
\end{equation}
By Lemma \ref{lem:2},
\[  \frac{1}{2}\sum_{t=1}^T\mathsf E_\theta\left( \sum_{j=1}^N\langle \widehat w_j,\Phi_{\theta_j}(x_t)\rangle-y_t\right)^2 \le T\frac{N}{m}\|f\|_{\mathcal H}^2+\frac{1}{2}\sum_{t=1}^T\left(f(x_t)-y_t\right)^2, \]
Combining (\ref{3.13}) with the last inequalities yields (\ref{3.8}). The relation (\ref{3.10}) follows immediately.
\end{proof}

Assume that $m\ge d$. Then a simple analysis shows that the time and space complexities of the proposed algorithm are $O(N^2 m^2)$ per iteration: see (\ref{2.4A}), (\ref{2.4B}). To reduce these complexities we consider
the following two-level procedure:
\begin{itemize}
\item generate $N$ $m$-dimensional vectors of random features, related to each kernel $k_i$, and apply VAW algorithm to each sequence $(\Phi_{\theta_j}(x_t),y_t)$,
\item regarding the predictions of these algorithms as expert opinions, combine them by a meta-algorithm. 
\end{itemize}

Our main suggestion is to use VAW also as a meta-algorithm. Assuming that $m\ge\max\{d,N\}$, the overall time and space complexities in this case are $O(N m^2)$ per iteration. This estimate reflects the complexities of the expert algorithms, as the meta-algorithm's contribution is negligible. The loss estimates are given in Theorem \ref{th:2}. Note that the justification of these estimates do not require the boundedness of the expert outputs $\langle w_{t,j},\Phi_{\theta_j}(x_t) \rangle$.

\begin{theorem} \label{th:2}
Let $w_{t,j}\in\mathbb R^m$ be generated by the VAW algorithms applied to $(\Phi_{\theta_j}(x_t),y_t)$, and $\alpha_t\in\mathbb R^N$ be generated by the VAW algorithm applied to $(z_t,y_t)$, where $z_t$ is the vector of expert predictions:
\[ z_t=(\langle w_{t,1},\Phi_{\theta_1}(x_t)\rangle,\dots,\langle w_{t,N},\Phi_{\theta_N}(x_t)\rangle).\]
Then
\begin{align}
 \frac{1}{2}\mathsf E_\theta\sum_{t=1}^T (\langle\alpha_t,z_t\rangle-y_t)^2 & \le \frac{1}{2}\sum_{t=1}^T (y_t-f_j(x_t))^2+\frac{\lambda}{2}+\left(\frac{\lambda}{2}+T\right)\frac{\|f_j\|_{\mathcal H_j}^2}{m}+ \frac{m Y^2 }{2} \ln \left(1+\frac{2T}{\lambda} \right)\nonumber\\
 & +\frac{N Y^2 }{2} \ln \left(1+\frac{Y^2}{\lambda}\left(2T(T+1)+ 2 m T \ln\left(1+\frac{2T}{\lambda}\right)\right)\right),\label{3.14} 
\end{align}
for any $f_j\in\mathcal H_j$, $j=1,\dots,N$. For $T\to+\infty$, 
\begin{align}
\frac{1}{2}\mathsf E_\theta\sum_{t=1}^T (\langle\alpha_t,z_t\rangle-y_t)^2 &\le \frac{1}{2}\min_{1\le j\le N}\inf_{f_j\in B_R(\mathcal H_j)}\sum_{t=1}^T (y_t-f_j(x_t))^2\nonumber\\
& + O\left((R^2+Y^2\ln T)\sqrt T \right),\quad \textrm{if}\quad m\propto \sqrt T. \label{3.16}   
\end{align}
\end{theorem}
\begin{proof} Take $f_j$, $\widehat w_j$ as in Lemma \ref{lem:2}. Denote by $R_T^{\operatorname{VAW}}(\delta)$ the regret of the VAW algorithm applied to the sequence $(z_t,y_t)$, and by $R_T^{\operatorname{VAW}}(\widehat w_j)$ the regret of the VAW algorithm applied to $(\Phi_{\theta_j}(x_t),y_t)$. For any $\delta_i\ge 0$, $\sum_{i=1}^N\delta_i=1$ we have
\begin{align}
 \frac{1}{2}\sum_{t=1}^T (\langle\alpha_t,z_t\rangle-y_t)^2 &=R_T^{\operatorname{VAW}}(\delta)+\frac{1}{2}\sum_{t=1}^T (\langle\delta,z_t\rangle-y_t)^2\le R_T^{\operatorname{VAW}}(\delta)+\frac{1}{2}\sum_{t=1}^T \sum_{j=1}^N\delta_j (z_{t,j}-y_t)^2  \nonumber\\ 
 & = R_T^{\operatorname{VAW}}(\delta)+\sum_{j=1}^N \delta_j R_T^{\operatorname{VAW}}(\widehat w_j)+\frac{1}{2}\sum_{t=1}^T \sum_{j=1}^N\delta_j (\langle \widehat w_j,\Phi_{\theta_j}(x_t)\rangle-y_t)^2 \label{3.18}
 \end{align}
 
 Let us estimate the first term. From the general bound (\ref{2.5}) for the regret of the VAW algorithm it follows that
 \begin{equation} \label{3.19}
 R_T^{\operatorname{VAW}}(\delta)\le \frac{\lambda}{2}\|\delta\|_2^2 + \frac{N Y^2 }{2} \ln \left(1+\frac{Z_T^2 T}{\lambda }\right),   
 \end{equation}
if $\sum_{j=1}^N z_{t,j}^2 = \sum_{j=1}^N \langle w_{t,j},\Phi_{\theta_j}(x_t)\rangle^2 \le Z_T^2$. Due to the logarithmic scaling of $Z_T$, its rough estimate would be enough.
We have
\begin{align*}
  z_{t,j}^2\le 2(\langle w_{t,j},\Phi_{\theta_j}(x_t)\rangle-y_t)^2+2y_t^2. 
\end{align*}
By the bound (\ref{2.5A}),
\begin{align*}
\frac{1}{2}(\langle w_{t,j},\Phi_{\theta_j}(x_t)\rangle-y_t)^2 &\le \frac{1}{2}\sum_{t=1}^T (\langle w_{t,j},\Phi_{\theta_j}(x_t)\rangle-y_t)^2=\frac{1}{2}\sum_{t=1}^T (\langle w_j,\Phi_{\theta_j}(x_t)\rangle-y_t)^2+R_T^{\operatorname{VAW}}(w_j)\\
&\le \frac{1}{2}\sum_{t=1}^T (\langle w_j,\Phi_{\theta_j}(x_t)\rangle-y_t)^2+\frac{\lambda\|w_j\|^2}{2}+\frac{m Y^2}{2}\ln\left(1+\frac{2T}{\lambda}\right)
\end{align*}
for any $w_j\in\mathbb R^m$.
Put $w_j=0$ in the right-hand side of the last formula:
\begin{align*}
\frac{1}{2}(\langle w_{t,j},\Phi_{\theta_j}(x_t)\rangle-y_t)^2 &\le \frac{1}{2}T Y^2 + 
\frac{m Y^2}{2}\ln\left(1+\frac{2T}{\lambda}\right).
\end{align*}
Thus,
\begin{equation} \label{3.20}
\sum_{j=1}^N z_{t,j}^2\le Z_T^2:=2(T+1) N Y^2 + 2 m N Y^2\ln\left(1+\frac{2T}{\lambda}\right).   
\end{equation}
From (\ref{3.19}), (\ref{3.20}) we get
\begin{equation} \label{3.21}
 R_T^{\operatorname{VAW}}(\delta)\le \frac{\lambda}{2}\|\delta\|_2^2 + \frac{N Y^2 }{2} \ln \left(1+\frac{Y^2}{\lambda}\left(2T(T+1)+ 2 m T \ln\left(1+\frac{2T}{\lambda}\right)\right)\right). 
\end{equation}

The estimate of the expectation of the second term in (\ref{3.18}) follows from (\ref{2.5A}) and Lemma \ref{lem:1}:  
\begin{equation} \label{3.22}
\sum_{j=1}^n\delta_j\operatorname*{\mathsf E}_\theta R_T^{\operatorname{VAW}}(\widehat w_j)\le \frac{\lambda}{2m}\sum_{j=1}^n\delta_j \|f_j\|_{\mathcal H_j}^2 + \frac{m Y^2 }{2} \ln \left(1+\frac{2T}{\lambda}\right).
\end{equation}
Finally, estimate the expectation of the last term in (\ref{3.18}) by Lemma \ref{lem:2} (applied with $N=1$):
\begin{align}
\frac{1}{2}\sum_{t=1}^T \sum_{j=1}^N\delta_j \mathsf E_\theta (\langle \widehat w_j,\Phi_{\theta_j}(x_t)\rangle-y_t)^2 &\le \sum_{j=1}^N\delta_j\left(\frac{T}{m}\|f_j\|^2_{\mathcal H_j}+\frac{1}{2}\sum_{t=1}^T (f_j(x_t)-y_t)^2\right)\label{3.23}.
\end{align}
To get (\ref{3.14}) consider the vectors of the standard basis $\delta=e_j$ of $\mathbb R^N$, and combine (\ref{3.18}), (\ref{3.21}), (\ref{3.22}) with (\ref{3.23}). The relation (\ref{3.16}) follows directly.
\end{proof}

Now assume that the upper bound $Y$ for $y_t$ is known. Then the last term in (\ref{3.14}) can be improved by changing expert predictions from $z_t$ to 
\begin{equation} \label{3.25}
\overline z_t=\min(Y,\max(z_t,-Y)),\quad z_{t,j}=\langle w_{t,j},\Phi_{\theta_j}(x_t)\rangle,
\end{equation}
where the $\max$ and $\min$ operations are applied component-wise.  Let $\overline R_T^{\operatorname{VAW}}(\delta)$ be the regret of the VAW algorithm applied to the sequence $(\overline z_t,y_t)$. Then
\begin{align*}
 \sum_{t=1}^T (\langle\alpha_t,\overline z_t\rangle-y_t)^2 &=\overline R_T^{\operatorname{VAW}}(\delta)+\sum_{t=1}^T (\langle\delta,\overline z_t\rangle-y_t)^2\le\overline R_T^{\operatorname{VAW}}(\delta)+\sum_{t=1}^T \sum_{j=1}^N\delta_j (z_{t,j}-y_t)^2
 \end{align*} 
for any $\delta_i\ge 0$, $\sum_{i=1}^N\delta_i=1$, since $(\overline z_{t,j}-y_t)^2\le (z_{t,j}-y_t)^2$. In (\ref{3.19}) we can put $Z_T=Y$:
\[ \overline R_T^{\operatorname{VAW}}(\delta) \le \frac{\lambda}{2}+\frac{NY^2}{2}\ln\left(1+\frac{Y^2 T}{\lambda} \right),\]
and use this bound instead of (\ref{3.21}).

Under the same assumption, the bounds (\ref{3.14}) can be further improved by 
using another algorithms for combining expert opinions, instead of VAW. Recall that a loss function $\ell:[-Y,Y]^2\to \mathbb R$ is called $\eta$-exponentially concave if the function $F(z) = e^{-\eta\ell(y,z)}$ is concave for all $y\in [-Y,Y]$. In particular, the loss the function $\ell(y,z)=(y-z)^2$ is $\eta$-exp-concave for $\eta\le 1/(8 Y^2)$ (see \cite[Section 3.3]{Cesa2006prediction}). Applying exponentially weighted average (EWA) forecaster: $\alpha_{1,j}=1/N$, 
\begin{equation} \label{3.25A}
 \alpha_{t,j}=\frac{\alpha_{t-1,j} \exp(-\eta (\overline z_{t,j}-y_t)^2 )}{\sum_{k=1}^N \alpha_{t-1,k} \exp(-\eta (\overline z_{t,k}-y_t)^2 )},\quad t=2,\dots, T
\end{equation}
with $\eta=1/(8 Y^2)$, we get the estimate 
\begin{equation} \label{3.26}
 \overline R_{T,j}^{\operatorname{EWA}}:=\frac{1}{2}\sum_{t=1}^T (\langle\alpha_t,\overline z_t\rangle-y_t)^2-\frac{1}{2}\sum_{t=1}^T (\overline z_{t,j}-y_t)^2\le 4 Y^2\ln N,  
\end{equation}
see \cite[Proposition 3.1]{Cesa2006prediction}. The related improved bounds are given in the next theorem.
\begin{theorem} \label{th:3}
Assume that the constant $Y$ is known. Let $w_{t,j}\in\mathbb R^m$ be generated by the VAW algorithms applied to the sequence $(\Phi_{\theta_j}(x_t),y_t)$, and $\alpha_t\in\mathbb R^N$ be generated by the EWA forecaster applied to the sequence $(\overline z_t,y_t)$, where $\overline z_t$ is the vector of truncated expert predictions (\ref{3.25}). Then
\begin{align}
 \frac{1}{2} \mathsf E_\theta\sum_{t=1}^T (\langle\alpha_t,\overline z_t\rangle-y_t)^2 &\le \frac{1}{2}\sum_{t=1}^T (f_j(x_t)-y_t)^2+ 4Y^2\ln N\nonumber\\
 &+\left(\frac{\lambda}{2}+T \right)\frac{\|f_j\|_{\mathcal H_j}^2}{m}
+ \frac{m Y^2 }{2} \ln \left(1+\frac{2 T}{\lambda}\right) \label{3.27}
\end{align}
for any $f_j\in\mathcal F_j$, $j=1,\dots,N$. The estimate (\ref{3.16}) of Theorem \ref{th:2} remains true.
\end{theorem}
\begin{proof} Using (\ref{3.26}), we get 
\begin{align*}
\frac{1}{2}\sum_{t=1}^T (\langle\alpha_t,\overline z_t\rangle-y_t)^2 &=\overline R_{T,j}^{\operatorname{EWA}} + \frac{1}{2}\sum_{t=1}^T (\overline z_{t,j}-y_t)^2\le 4 Y^2\ln N + \frac{1}{2}\sum_{t=1}^T (\langle w_{t,j},\Phi_{\theta_j}(x_t)\rangle-y_t)^2\\
&\le 4 Y^2\ln N+ R_T^{\operatorname{VAW}}(\widehat w_j)+\frac{1}{2}\sum_{t=1}^T (\langle \widehat w_j,\Phi_{\theta_j}(x_t)\rangle-y_t)^2.
\end{align*}
The assertion follows from this inequality combined with the estimates (\ref{3.22}), (\ref{3.23}) applied to $\delta=e_j$.
\end{proof}

Let us call the algorithms, analyzed in Theorems \ref{th:2} and \ref{th:3} by VAW$^2$ (double VAW) and VAW-EWA respectively. Under the mentioned assumption $m\ge\max\{d,N\}$ their time and space complexities are the same: $O(Nm^2)$, and are determined by the complexities of the expert VAW algorithms. 

Although the estimate (\ref{3.27}) is slightly better than (\ref{3.14}), the estimate (\ref{3.16}) for large $T$ in Theorem \ref{th:3} is not improved.
It is not clear if the improvement, obtained by applying the EWA forecaster to the truncated expert opinions $\overline z_t$ instead of applying VAW algorithm to original expert opinions $z_t$, is essential. The numerical experiments, presented in Section \ref{sec:4}, show that the linear combinations of expert predictions, used in VAW$^2$, can produce better results than the convex combinations of the VAW-EWA or similar meta-algorithms.

Using a similar notation, the basic algorithms used in  \cite{Sahoo2019large,Shen2019random} can be called OGD-OGD and OGD-EWA, since they use the online gradient descent (OGD) for expert strategies, and either OGD or EWA-type meta-algorithms. Their time and space per iteration complexities are lower: $O(Ndm)$. However the related regret bounds are not applicable, since the quadratic loss function does not satisfy the global Lipschitz condition, and coefficients $\widehat w$ in Lemma \ref{lem:1} are not bounded.

\section{Computer experiments} \label{sec:4}
In the organization of computer experiments we followed \cite{Ghari2023graph} and the related code\footnote{\url{https://github.com/pouyamghari/Graph-Aided-Online-Multi-Kernel-Learning}}.
Code to reproduce our results is available at\footnote{\url{https://github.com/O-Gurt/VAW2}}, along with instructions for running the experiments.
All algorithms were run using $N=76$ kernels: 51 Gaussian and 25 Laplacian: see (\ref{2.2A}), (\ref{2.2B}). Their parameters were set as follows:
\[ \sigma^2\in\{10^{2i/25-2}\}_{i=0}^{50}\quad \textrm{for Gaussian kernels},\qquad \sigma\in\{10^{i/6-2}\}_{i=0}^{24}\quad \textrm{for Laplacian kernels}. \]

Following  \cite{Ghari2023graph} we used random features of the form $(\cos\langle\theta_i,x\rangle,\sin\langle\theta_i,x\rangle)$, $\theta_i\sim q$, $i=1,\dots,m$. This is a well-known slight variation of the approach described above (see  \cite{Sutherland2015error} for a discussion). It is related to the kernel representation
\begin{align*}
\int_\Theta p(\theta)\langle\phi(x;\theta) \phi(y;\theta)\rangle\,d\theta = \int_{\mathbb R^d} \cos\langle\theta,x-y\rangle q(\theta) d\theta=\kappa(x-y)=k(x,y).
\end{align*}

The number $m$ of random features was set to 50. The mean squared losses (MSE) $\frac{1}{T}\sum_{t=1}^T (\widehat f_t(x_t)-y_t)^2 $ of various algorithms $\widehat f_t$ were averaged over 5 experiments.
We chose $\lambda=1$ for the VAW algorithm in all cases.

Furthermore, we used the same datasets as in \cite{Ghari2023graph}. They are briefly described in Table \ref{tab:1} and are available from the UCI Machine Learning Repository\footnote{\url{https://archive.ics.uci.edu/}}. In addition we generated an artificial data by AR(4) model:
\begin{equation}
    x_t = \nu_0 x_{t-4} + \nu_1 x_{t-3} + \nu_2 x_{t-2} + \nu_3 x_{t-1} + \epsilon_t,\quad y_t=x_{t+1}, \quad t=1,\dots,5000,
\end{equation}
where $\nu_0 = 0.5$, $\nu_1 = -0.3, \nu_2 = 0.2, \nu_3 = 0.1$, $\epsilon_t \sim \mathcal{N}(0, 1)$, $x_k=0$, $k=-3,\dots,0$.   

As in the mentioned code of \cite{Ghari2023graph}, for all datasets the features and labels were normalized as follows:
\begin{gather}
    y_i := \frac{y_i - \underline y}{\overline{y} - \underline{y}},\quad \underline y=\min_{j=1,\dots,n} y_j,\quad \overline y=\max_{j=1,\dots,n} y_j,\\
    x_i:=x_i/\max_{j=1,\dots,n}\|x_j\|_2.
\end{gather}

\begin{table}[h] 
    \centering    
    \begin{tabular}{|c| c| p{5cm}| p{5cm}|}
        \hline
        Name & Size & Data description &  Label \\
        \hline
        Airfoil & (1503, 5) & airfoils at various wind tunnel speeds and angles of attack &  scaled sound pressure \\
        \hline        
        Bias & (7750, 21) & temperature measurements and predictions together with auxiliary geographic variables  & next-day minimum air temperature \\
        \hline
        Concrete & (1030, 8) & concrete specifications such as the amount of cement or water   & compressive strength \\
        \hline
        Naval  & (11934, 15) & features of a naval vessel, characterized by a gas turbine propulsion plant  & lever position \\
        \hline
    \end{tabular}    
    \caption{Summary of real-world datasets used for evaluation.}
    \label{tab:1}
   \end{table} 

We compared the VAW$^2$ algorithm, analyzed in Theorem \ref{th:2}, and the VAW-EWA algorithm, analyzed in Theorem \ref{th:3}, with several other algorithms:  
\begin{itemize}
\item {Raker \cite{Shen2019random}:} this is the algorithm of OGD-EWA type in our notation. It combines OGD the predictions of expert strategies by the EWA-type meta-algorithm. 
\item{OMKL-GF \cite{Ghari2023graph}}: a data-driven kernel selection scheme where a bipartite feedback graph is constructed at every time instant. 
\item{VAW-Aggr:} the predictions of VAW expert strategies are combined by the Vovk VAW-Aggregating algorithm \cite[Section 3.5]{Cesa2006prediction}. The quadratic loss is $\eta$-mixable with $\eta=2$ \cite[Section 3.6]{Cesa2006prediction}. Thus, using the VAW-Aggregating meta-algorithm with $\eta=2$, it is possible to achieve the regret estimate slightly better then for the EWA meta-algorithm \cite[Proposition 3.2]{Cesa2006prediction}.
\item{VAW-ML-Prod, VAW-ML-Poly, VAW-BOA}: the predictions of VAW expert strategies are combined by second-order online algorithms, which use both the cumulative loss (first-order statistic) and the variance of losses (second-order statistic) to adapt their learning rates dynamically \cite{Gaillard2014second,Wintenberger2017optimal}. 
These algorithms are implemented within the Opera library\footnote{\url{https://github.com/Dralliag/opera-python}}, which we employed.
\end{itemize}
We do not consider the VAW algorithm from Theorem \ref{th:1} due to its high computational and space complexities.

The results of experiments are collected in Table \ref{tab:3}.
We do not describe here the parameters of Raker and OMKL-GF algorithms. The results of \cite{Ghari2023graph} were reproduced by running their publicly available code with the parameters they specified. While \cite{Ghari2023graph} averaged results over 20 experiments, we used 5. So, the results presented here are slightly different. Note that theoretically all these algorithms, except VAW$^2$, require knowledge of the interval containing the labels, and should be used with the truncated expert predictions. Since here we consider $y_t\in [0,1]$, instead of $y_t\in [-Y,Y]$, the truncation was performed accordingly:
\[ \overline z_t=\min(1,\max(z_t,0)),\quad z_{t,j}=\langle w_{t,j},\Phi_{\theta_j}(x_t)\rangle. \]
For the VAW$^2$ algorithm we present the results both for original and truncated expert predictions: VAW$^2$(trunc). However, these options give almost the same results. The lowest MSE values are shown in bold. VAW$^2$ algorithm shows the best result across all datasets.
\begin{table}[!ht]
    \centering
    \makebox[\textwidth]{ 
    \begin{tabular}{c|c|c|c|c|c}
        \hline
         & AR(4) & Airfoil & Bias & Concrete & Naval \\ \hline
   Raker & 23.24 & 28.64 & 12.70 & 35.29  & 11.32 \\ \hline
OMKL-GF  & 20.47 & 24.37 & 7.05 & 34.24  & 4.60 \\ \hline
VAW$^2$  &  16.56 & 22.80 & \textbf{4.09} & \textbf{10.96} & \textbf{0.29} \\ \hline
VAW$^2$(trunc) &  16.51 & \textbf{22.78} & \textbf{4.09} & 10.97 & \textbf{0.29} \\ \hline
VAW-Aggr &  16.40 & 26.74 & 5.02 & 13.57 &  0.45 \\ \hline
VAW-EWA &  16.49 & 27.61 & 5.41 & 15.08 & 0.62 \\ \hline
VAW-BOA &  \textbf{16.34} & 26.42 & 4.98 & 13.88 & 0.52 \\ \hline
VAW-ML-Poly &  \textbf{16.34} & 26.10 & 4.96 & 13.33 & 0.37 \\ \hline
VAW-ML-Prod &  \textbf{16.34}  & 26.27 & 4.97 & 13.64 &  0.48 \\ \hline
    \end{tabular}}
    \caption{ MSE (scaled up by $10^3$) of MKL algorithms with 76 kernels. }
    \label{tab:3}
\end{table}

Figure \ref{fig:1} illustrates the MSE of these algorithms over the iterations. We excluded VAW$^2$(trunc), VAW-BOA, AW-ML-Poly to improve the clarity. VAW$^2$ consistently achieves the lowest MSE trajectory across considered real world datasets, indicating strong performance throughout learning. 

\begin{figure}[h!] 
  \centering
  \begin{subfigure}[t]{0.45\textwidth}
    \includegraphics[width=\textwidth]{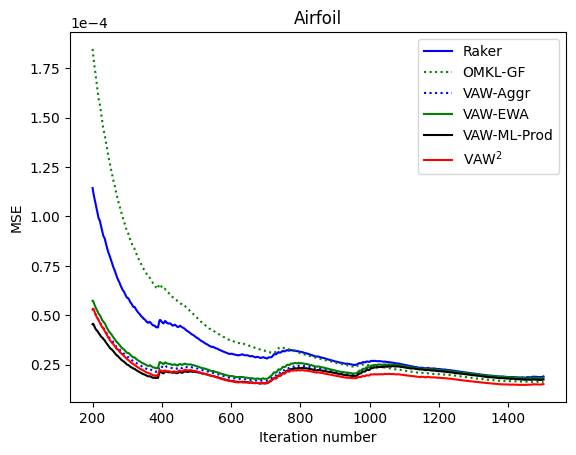}
  \end{subfigure}
  \hfill 
  \begin{subfigure}[t]{0.45\textwidth}
    \includegraphics[width=\textwidth]{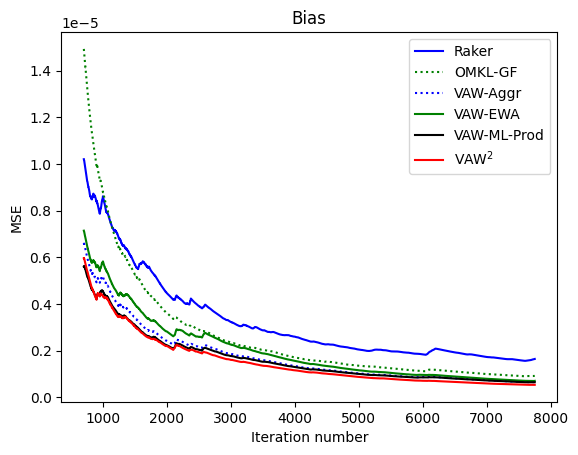}
  \end{subfigure}

  \begin{subfigure}[t]{0.45\textwidth}
    \includegraphics[width=\textwidth]{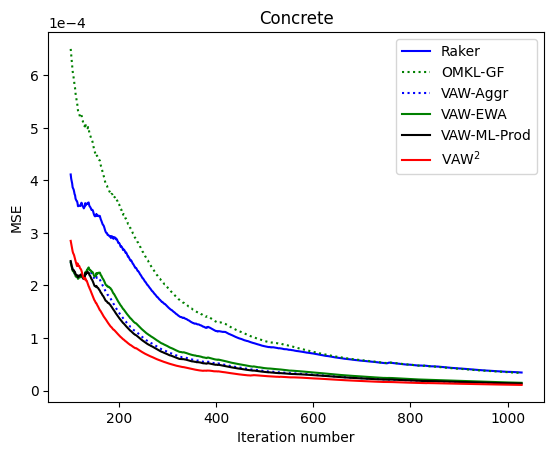}
  \end{subfigure}
  \hfill
  \begin{subfigure}[t]{0.45\textwidth}
    \includegraphics[width=\textwidth]{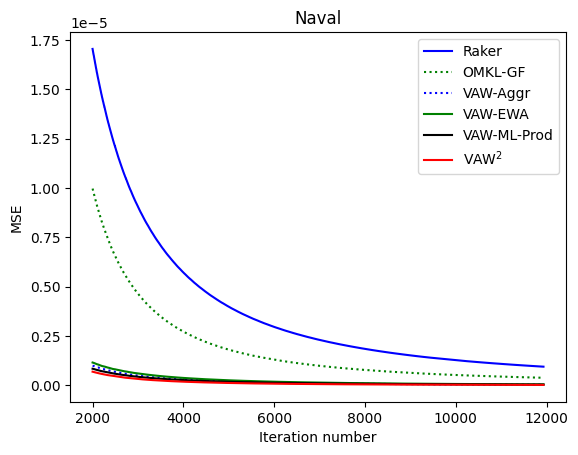}
  \end{subfigure}
  \caption{MSE performance of MKL algorithms.}
  \label{fig:1}
\end{figure}

To further understand the behavior of the suggested algorithms, in Figure \ref{fig:2} we plot the final expert weight vectors $\alpha_T$, assigned by VAW$^2$, VAW-EWA and VAW-ML-Prod algorithms. 
We see that ML-prod exhibits sparsity, concentrating its weighting on a small number of kernels. EWA distributes weight more broadly, while
VAW distinguishes itself by the essential use of negative weights.

\begin{figure}[h!] 
  \centering
  \begin{subfigure}[t]{0.45\textwidth}
    \includegraphics[width=\textwidth]{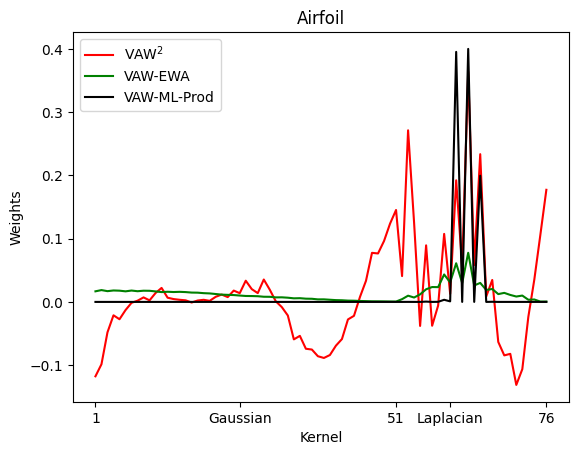}
  \end{subfigure}
  \hfill 
  \begin{subfigure}[t]{0.45\textwidth}
    \includegraphics[width=\textwidth]{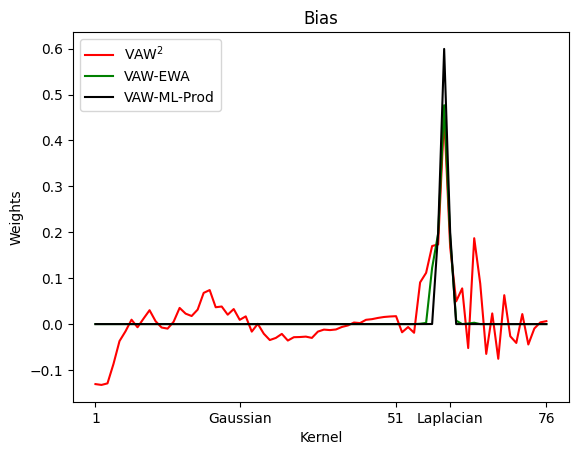}
  \end{subfigure}

  \begin{subfigure}[t]{0.45\textwidth}
    \includegraphics[width=\textwidth]{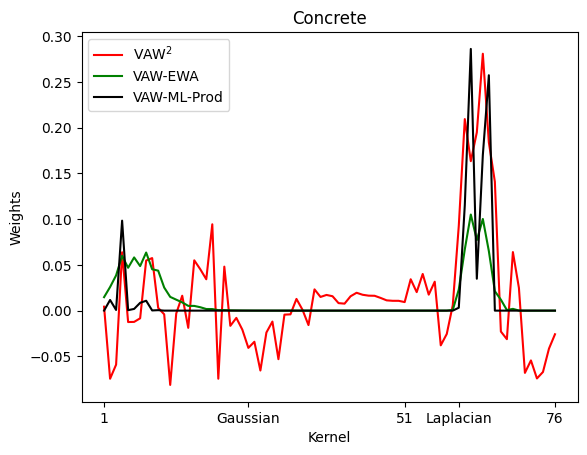}
  \end{subfigure}
  \hfill
  \begin{subfigure}[t]{0.45\textwidth}
    \includegraphics[width=\textwidth]{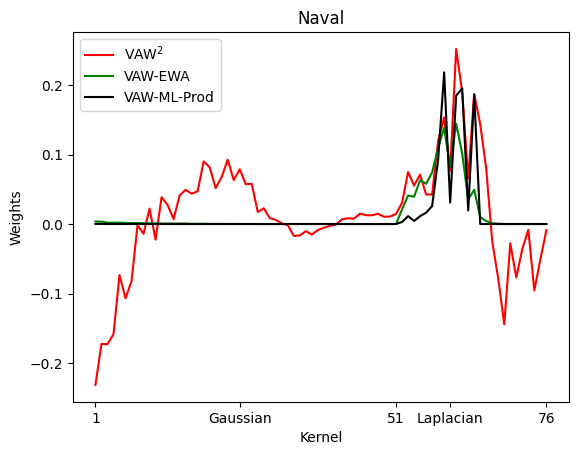}
  \end{subfigure}
  \caption{Final weights of VAW$^2$, VAW-EWA and VAW-ML-Prod algorithms.}
  \label{fig:2}
\end{figure}

\section{Conclusion} \label{sec:5} 
We introduced VAW$^2$, a novel online multi-kernel learning algorithm for least squares regression in RKHS. By leveraging VAW at both the expert level (for kernel-specific predictions) and the meta level (for dynamic kernel combination), VAW$^2$ achieves a balance between computational efficiency and theoretical guarantees.
A key feature of VAW$^2$ is its computational efficiency compared to direct application of the VAW algorithm to concatenated feature vectors, making it scalable for practical applications. We derived a regret bound of $O(T^{1/2}\ln T)$ in expectation with respect to artificial randomness, when the number of random features scales as $T^{1/2}$. The framework accommodates both VAW and EWA meta-algorithms, with truncation strategies further enhancing robustness when label bounds are known. Computational experiments showed encouraging results on some benchmark datasets.

Future work could extend this analysis to derive dynamic regret bounds for non-stationary environments, incorporate mechanisms for online kernel dictionary adaptation, and refine loss bounds under specific data assumptions. It is interesting to perform more extensive benchmarking of the VAW$^2$ algorithm across diverse datasets and application domains.

\printbibliography

\end{document}